\documentclass[conference]{IEEEtran}
\IEEEoverridecommandlockouts

\usepackage{cite}
\usepackage{url}
\usepackage{amsmath,amssymb,amsfonts}
\usepackage{amsthm}
\usepackage{graphicx}
\usepackage{textcomp}
\usepackage{xcolor}
\usepackage{booktabs}
\usepackage{array}
\usepackage{tikz}
\usetikzlibrary{arrows.meta,positioning,calc,fit,backgrounds,shapes.geometric}

\newcommand{\Pmax}{P_{\max}}
\newcommand{\noise}{\sigma^2}
\newcommand{\Kq}{K_q}
\newcommand{\bp}{\mathbf{p}}
\newcommand{\bH}{\mathbf{H}}
\newcommand{\bgam}{\boldsymbol{\gamma}}
\newcommand{\bw}{\mathbf{w}}
\newcommand{\bF}{\mathbf{F}}
\newcommand{\btau}{\boldsymbol{\tau}}
\newcommand{\bff}{\mathbf{f}}
\newcommand{\own}[2]{\left|h_{#1,#2,#2}\right|^{2}}
\newcommand{\gain}[3]{\left|h_{#1,#2,#3}\right|^{2}}

\newtheorem{proposition}{Proposition}

\makeatletter
\def\@IEEENORMtitlevspace{0.45\baselineskip}
\def\@IEEEMINtitlevspace{0.3\baselineskip}
\makeatother
\def\IEEEtitletopspace{0.15\baselineskip}

\makeatletter
\renewcommand{\@listI}{\leftmargin\parindent \topsep 2pt \parsep 0pt \itemsep 2pt}
\makeatother
\definecolor{ink}{HTML}{16242C}
\definecolor{muted}{HTML}{5C6C75}
\definecolor{exactc}{HTML}{0E7C86}
\definecolor{learnc}{HTML}{B07D1E}
\definecolor{attnc}{HTML}{9E2B3F}

\begin{document}

\title{Agentic Autoresearch for Cell-Edge Power\\
Control: Radically Redefining the Researcher's Role}

\author{\IEEEauthorblockN{Ahmad Khan\textsuperscript{$\dagger$}
\qquad Akram Bin Sediq\textsuperscript{$\dagger$}
\qquad Sara Azadegi Naeini\textsuperscript{$\dagger\ddagger$}
\qquad Raviraj S. Adve\textsuperscript{$\ddagger$}}
\IEEEauthorblockA{\textsuperscript{$\dagger$}\,Ericsson R\&D, Ottawa, 349 Terry Fox Drive, ON K2K 2V6, Canada\\
\textsuperscript{$\ddagger$}\,ECE Dept., University of Toronto,\\
10 King's College Road, Toronto, ON M5S 3G4, Canada}
}

\maketitle
\begin{abstract}
Designing machine learning algorithms for wireless resource management is
labour-intensive: the architecture, the loss function and the training recipe
are all specified by hand. We demonstrate that this design layer can be
surrendered to an autonomous agent in its entirety. We adopt the autoresearch
protocol, in which an AI coding agent edits a training script, runs a
fixed-budget experiment, and retains or discards the change according to a
single immutable metric. We grant the agent authority over the architecture
family, the input representation, the output parameterization, the loss function
and the task-sampling law, and set it a target chosen for its difficulty:
sum-least-percentile-rate power control across a multicell network. The formulation targets cell-edge throughput and is non-convex, non-smooth and
strongly NP-hard away from its max-min vertex. Safeguards render the results trustworthy: a hash-pinned evaluator, an
enforced inference contract and a pre-registered falsifier per experiment. In eighty-one
unattended experiments over twenty-six hours, the agent reached $99.5\%$ of a
converged minorization-maximization reference in one fixed-cost inference pass, at
roughly $600\times$ lower inference cost, closing $94\%$ of the gap from its first working architecture, with one
parameter set serving every network size and percentile target. It recovered
provable structure rather than tuned constants: the output parameterization it
discovered reproduces the exact max-min-optimal allocation at the minimum
percentile, for every value of the trained weights.
\end{abstract}

\begin{IEEEkeywords}
Autoresearch, AI agents, percentile optimization, power control, cell-edge,
max-min fairness, learning to optimize.
\end{IEEEkeywords}

\section{Introduction}
\label{sec:intro}

\subsection{Background and Overview}

Automating the loop of hypothesis, modification, measurement and retention is
not new: robot scientists \cite{king} and, within machine learning, AutoML and
neural architecture search \cite{zoph} predate large language models, though
they search a space fixed in advance by the designer. The qualitative shift came once LLMs could propose \emph{code}
rather than parameters \cite{funsearch,aiscientist,alphaevolve}. Karpathy \cite{karpathy} distilled the pattern into the minimal form we adopt:
three files suffice, an \emph{immutable evaluator}, a single \emph{mutable
training script}, and a natural-language research charter. An agent edits the
script, runs a fixed-budget experiment, reads a single scalar metric, and
commits or reverts through version control, unattended. The human contracts to a
research director, who specifies the problem, the metric and the protocol, and
thereafter audits the log.

Agentic algorithm design has reached wireless at both link level, where
\cite{aitelco} evolves channel estimators and link adaptation against an
immutable evaluation tool, and network level, where \cite{allstar} synthesizes
MAC schedulers from the literature for over-the-air deployment, \cite{genesis}
drives the RAN engineering life-cycle end to end, and \cite{comagent} generates
beamforming formulations. In each, the artifact under autonomous design is
conventional code, or a policy trained inside a pipeline a human specified.

Here the agent redesigns the \emph{learned system} itself, holding authority
over the architecture family, the input representation, the output
parameterization, the loss function and the task-sampling law. It exercised all
five over six architecture families; Table~\ref{tab:milestones} records the
result. To the best of our knowledge no prior work in wireless grants more than
one or two. Our target, moreover, is optimizing an entire multicell network, and
the problem a provably NP-hard resource allocation coupled through
interference \cite{parti}. Others employ LLMs at run time, as solvers \cite{llmoptira,noh} or workflow
designers \cite{wirelessagent,automas}, not as algorithm designers.

Radio resource management (RRM) is a natural setting for autoresearch, supplying
the three ingredients the loop requires: a scalar figure of merit, a fast simulator to serve as judge, and a mutable algorithmic artifact. Learning to optimize resource allocation is well studied \cite{sun,eisen,shen},
including a hand-designed learner for the same sum-least-$q$th-percentile (SLqP)
objective \cite{selfsup}, which trains a separate model per network size and
percentile on networks without cellular wraparound; here one parameter set spans
both, across seven wrapped cells. In every such work, however, the architecture, the loss
function and the training recipe remain products of human judgement. Autoresearch automates exactly this design layer, yielding \emph{two nested
loops}: an inner loop in which gradient descent fits the weights of the current
model, and an outer loop in which an agent redesigns the model, its features and
its objective.

To stress test this thesis we select a deliberately difficult target,
\emph{percentile optimization}. Part~I of \cite{parti} formulates SLqP rate maximization, directly optimizing
the throughput of the weakest percentile of users; this is the cell-edge metric
for which 3GPP and other industrial bodies have set explicit
next-generation targets, and which prior physical-layer techniques have
addressed only indirectly. It is further established in \cite{parti} that, away
from its polynomial-time max-min vertex, the SLqP program is non-convex,
non-smooth and strongly NP-hard, with the state of the art being iterative
minorization-maximization. An amortized solution is accordingly of practical interest, per-instance solvers
being far too slow for real scheduling timescales, and of theoretical interest,
the objective coupling users through an order statistic that defeats
conventional pooling.

\subsection{Contributions of This Paper}

The contributions of this paper may be summarized as follows.

\begin{itemize}
\item \textbf{Protocol:} We harden the autoresearch loop for scientific use in a
physical-layer setting (Section~\ref{sec:campaign}). Where prior agentic work in
wireless designs symbolic code or tunes a human-specified pipeline, here every
layer of a learned system is under autonomous design, so the agent's mandate
extends to the inductive bias itself.
\item \textbf{Campaign:} Eighty-one unattended experiments over twenty-six hours
closed $94\%$ of the gap between the agent's first working architecture and a
converged reference solver, at roughly $600\times$ lower inference cost, with a
single controller
serving every $K$ and in-band percentile (Section~\ref{sec:results}).
\item \textbf{Discovered structure:} The campaign recovers interpretable
structure and not merely tuned hyperparameters: the output parameterization it
discovers pins the model to the \emph{exact} max-min optimum at the minimum
percentile, for every value of the trained weights
(Proposition~\ref{prop:clamp}), a property the agent itself identified and
exploited.
\item \textbf{Reproducibility and public release:} We release the relevant files,
the complete experiment log, and weights as well as scripts needed to reproduce the
champion.\footnote{\scriptsize\url{https://github.com/akhan-ericsson/autoresearch-percentile-optimization}}
\end{itemize}

Section~\ref{sec:problem} states the problem, Section~\ref{sec:campaign} the
protocol, Section~\ref{sec:solution} the discovered solution,
Section~\ref{sec:results} the results, and Section~\ref{sec:conclusion} the
lessons learned for others adopting the protocol.

\section{Percentile Power Control}
\label{sec:problem}

We consider the downlink of a network of $B$ mutually interfering cells, each
containing $K$ single-antenna users, so that the network serves $KB$ users in
total. We denote by $h_{k,b,b'}$ the channel from the transmitter of cell $b'$
to user $k$ of cell $b$, collected in the tensor $\bH$, and by
$p_{k,b}\in[0,\Pmax]$ the power allocated to user $(k,b)$, collected in $\bp$.
With $P_{b'}(\bp)=\sum_k p_{k,b'}$ denoting the total power emitted by cell
$b'$, the SINR achieved by user $(k,b)$ is
\begin{equation}
\gamma_{k,b}(\bp) = \frac{p_{k,b}\own{k}{b}}
{\sum_{b'} \gain{k}{b}{b'} P_{b'}(\bp) - p_{k,b}\own{k}{b} + \noise},
\label{eq:sinr}
\end{equation}
and the corresponding rate is $r_{k,b}(\bp)=\log_2(1+\gamma_{k,b}(\bp))$, where
$\noise$ is the receiver noise power. The subtraction removes the user's own signal from its cell total, while the
intra-cell interference of its $K-1$ cell-mates is retained. Following
\cite{parti}, $x^{\uparrow}_i$ denotes the $i$th smallest entry of $x$, so the
SLqP utility is $f_{\Kq}(x)=\sum_{i=1}^{\Kq}x^{\uparrow}_i$ and the
power-control problem is
\begin{equation}
\operatorname*{maximize}_{\,0\le p_{k,b}\le \Pmax}\;
f_{\Kq}\big(r_{1,1}(\bp),\dots,r_{K,B}(\bp)\big),
\label{eq:problem}
\end{equation}
in which the percentile number $\Kq=\lceil q\,KB\rceil$, for a percentile target
$q$ expressed as a fraction, counts the weakest rates the objective sums. This work addresses the cell-edge band $\Kq/(KB)\in(0,0.25]$. At
$\Kq=1$, \eqref{eq:problem} reduces to max-min-rate power control, which may be
solved to global optimality in polynomial time \cite{parti}; for every $\Kq>1$
it is non-convex and strongly NP-hard \cite{parti,luozhang}, and for
$1<\Kq<KB$ the objective is, in addition, non-smooth. Rather than solve each instance iteratively, we seek a single amortized map
$\bp_\theta:(\bH,\Kq)\mapsto\bp$, one fixed-cost inference pass serving every $K$
and in-band percentile, and task an agent with its discovery.

\section{The Autoresearch Campaign}
\label{sec:campaign}

\subsection{Protocol}
\label{sec:protocol}

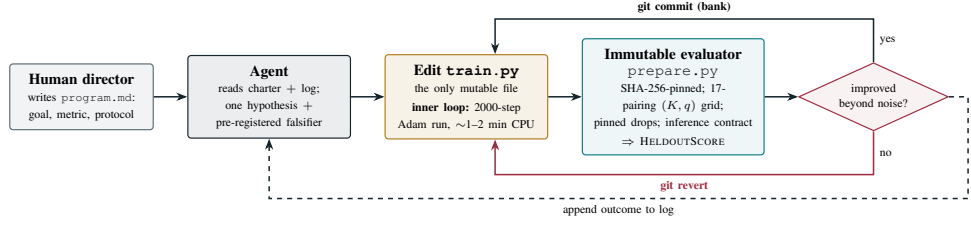
\begin{figure*}[t]
\centering
\resizebox{0.70\textwidth}{!}{%
\begin{tikzpicture}[
  font=\small,
  box/.style={draw, rounded corners=2pt, align=center, inner sep=2mm, line width=0.7pt},
  agent/.style={box, draw=ink, fill=ink!8, text width=30mm},
  mut/.style={box, draw=learnc, fill=learnc!10, text width=30mm},
  ev/.style={box, draw=exactc, fill=exactc!7, text width=34mm},
  dec/.style={draw=attnc, fill=attnc!6, diamond, aspect=2.2, align=center, inner sep=0.8mm, line width=0.7pt},
  hum/.style={box, draw=muted, fill=muted!8, text width=26mm},
  flow/.style={-{Stealth[length=2.3mm]}, line width=0.9pt, draw=ink},
  back/.style={-{Stealth[length=2.3mm]}, line width=0.9pt, draw=attnc}
]
\node[hum] (dir) {\textbf{Human director}\\{\scriptsize writes \texttt{program.md}:}\\{\scriptsize goal, metric, protocol}};
\node[agent, right=7mm of dir] (hyp) {\textbf{Agent}\\{\scriptsize reads charter $+$ log;}\\{\scriptsize one hypothesis $+$}\\{\scriptsize pre-registered falsifier}};
\node[mut, right=7mm of hyp] (edit) {\textbf{Edit \texttt{train.py}}\\{\scriptsize the only mutable file}\\[1.5pt]{\scriptsize\textbf{inner loop:} 2000-step}\\{\scriptsize Adam run, ${\sim}$1--2 min CPU}};
\node[ev, right=7mm of edit] (ev) {\textbf{Immutable evaluator} \texttt{prepare.py}\\{\scriptsize SHA-256-pinned; 17-pairing $(K,q)$ grid;}\\{\scriptsize pinned drops; inference contract}\\[1.5pt]{\scriptsize$\Rightarrow$ \textsc{HeldoutScore}}};
\node[dec, right=7mm of ev] (dec) {\scriptsize improved\\[-1pt]\scriptsize beyond noise?};

\draw[flow] (dir) -- (hyp);
\draw[flow] (hyp) -- (edit);
\draw[flow] (edit) -- (ev);
\draw[flow] (ev) -- (dec);
\coordinate (topr) at ($(ev.north)+(0,4mm)$);
\coordinate (botr) at ($(ev.south)+(0,-4mm)$);
\coordinate (logr) at ($(ev.south)+(0,-9mm)$);
\coordinate (edN) at ($(edit.north)+(6mm,0)$);
\coordinate (edS) at ($(edit.south)+(6mm,0)$);
\coordinate (dE)  at ($(dec.east)+(4mm,0)$);
\draw[flow] (dec.north) -- (dec.north |- topr)
  node[pos=0.5, right, font=\scriptsize]{yes}
  -- node[above, font=\scriptsize\bfseries]{git commit (bank)}
  (edN |- topr) -- (edN);
\draw[back] (dec.south) -- (dec.south |- botr)
  node[pos=0.5, right, font=\scriptsize]{no}
  -- node[below, font=\scriptsize\bfseries, text=attnc]{git revert}
  (edS |- botr) -- (edS);
\draw[flow, dashed] (dec.east) -- (dE) -- (dE |- logr)
  -- node[below, font=\scriptsize]{append outcome to log}
  (hyp.south |- logr) -- (hyp.south);
\end{tikzpicture}}
\caption{The two nested loops of the campaign.}
\label{fig:loop}
\end{figure*}

Following \cite{karpathy}, the campaign is built from three files with strictly
separated roles, as illustrated in Fig.~\ref{fig:loop}.

\begin{itemize}
\item \texttt{prepare.py}, the \textbf{immutable evaluator}, implementing the
channel model of \cite{parti} -- seven wrapped hexagonal cells, COST231 path
loss and Rayleigh fading, certified against that work's reference
implementation -- and scoring any candidate on a fixed, pinned held-out set. Its
SHA-256 hash is verified at every iteration; the agent cannot edit this judge.
\item \texttt{train.py}, the \textbf{sole file the agent may edit}, containing
the model, the loss and the training loop, and accumulating an append-only
changelog within its own docstring.
\item \texttt{program.md}, the \textbf{research charter}, stating the goal, the
metric, the reference bars, the inference contract and the exploration
protocol.
\end{itemize}

The metric is designed for comparability across settings. The held-out benchmark
is a grid of seventeen $(K,\text{percentile})$ pairings, with
$K\in\{1,2,4,6,8,10\}$ and targets $\{\min,\,p_{10},\,p_{25}\}$. Since raw SLqP
values differ by orders of magnitude depending on the values of $K$ and $\Kq$,
each pairing is scored as the ratio
of the model's mean SLqP to the \emph{full-power} mean SLqP on the same pinned
realizations, and \textsc{HeldoutScore} is the mean of these ratios. A score of $1.000$ therefore
denotes the trivial full-power floor, while the converged minorization-maximization algorithm (based on the iterative quadratic fractional transform (QFT))
reference of \cite{parti}, measured sample-matched on identical realizations,
sets the bar at $1.485$. Because the agent commits or reverts against it, this score is a
selection set; the champion was re-scored once on a disjoint, independently
seeded set of realizations, on which the reported figures stand. The evaluator additionally enforces an \emph{inference contract}: a ten-second
budget for inference over the entire grid, precluding disguised per-instance
optimization at test time, together with a no-test-time-fitting tripwire and
output-shape guards. A violation raises an exception rather than a score, and so cannot be banked.

The outer loop thereafter proceeds unattended. In each cycle the agent reads the
charter and the accumulated log, states a \emph{single} hypothesis with a \emph{pre-registered falsifier}, the
observation that would disprove it, edits \texttt{train.py}, and launches the inner loop: a complete training
run over $2000$ Adam steps, requiring one to two minutes on the same CPU, followed by
evaluation on the pinned grid. If \textsc{HeldoutScore} improves beyond a noise band of $\pm0.0005$, calibrated
by repeated identical runs, the change is committed; otherwise it is rejected.
Either way the outcome, including whether the falsifier fired, is appended to
the log before the next cycle.

Exploration breadth is governed by the charter specified in \texttt{program.md} rather than left to chance: by
deliberate choice, at most six architecture families may be opened, and a newly
opened family is protected from reversion for a fixed number of experiments, so that a new
idea is not discarded on a single untuned attempt. A family that has not
overtaken the incumbent by the end of that window is abandoned, and the
remaining budget is directed to depth on the most promising family. This protocol was
itself learned from an earlier full-range campaign, whose two failure modes,
seventeen families opened with one experiment each and none tuned, and a single
model straddling the policy shift near the sum-rate end, motivated the breadth cap and the restriction to the cell-edge band.

\subsection{On the Necessity of the Safeguards}

Each element earns its place. The hash pin removes, by construction, the characteristic failure of
self-improving systems, that of making the test easier rather than the model
better; other agentic-search systems arrived at the same principle independently
\cite{alphaevolve}. The inference contract keeps the search aligned with the engineering objective:
since an amortized model exists to improve upon solver latency, a candidate that
covertly optimizes per instance must fail rather than score. Pre-registered falsifiers convert the log into a sequence of
interpretable findings, negative results included. Sample-matching eliminates realization-to-realization variance; a mid-campaign
audit at experiment~63 corrected a violation of that discipline, altering the
interpretation of several earlier results.

\section{The Discovered Solution}
\label{sec:solution}

We now summarize the champion model to which the campaign converged, whose
pipeline is set out in Fig.~\ref{fig:pipeline}; full architectural detail is in
the released repository. The design that emerged is notable for the
quantity of exact classical structure it contains: only the encoder and a scalar
output head are trained, and the remainder is closed-form algebra that the agent
progressively moved \emph{out} of the learned components.

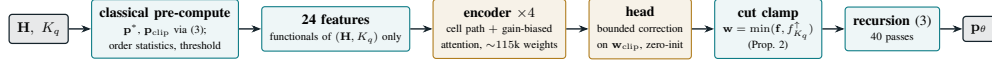
\begin{figure*}[t]
\centering
\resizebox{0.72\textwidth}{!}{%
\begin{tikzpicture}[
  font=\small,
  ex/.style={draw=exactc, fill=exactc!7, rounded corners=2pt, align=center, inner sep=1.8mm, line width=0.7pt},
  ln/.style={draw=learnc, fill=learnc!10, rounded corners=2pt, align=center, inner sep=1.8mm, line width=0.7pt},
  io/.style={draw=ink, fill=ink!10, rounded corners=2pt, align=center, inner sep=1.8mm, line width=0.8pt},
  fl/.style={-{Stealth[length=2.2mm]}, line width=0.9pt, draw=ink}
]
\node[io] (in) {$\bH,\ \Kq$};
\node[ex, right=5mm of in] (pre) {\textbf{classical pre-compute}\\{\scriptsize $\bp^{*}$, $\bp_{\mathrm{clip}}$ via \eqref{eq:recursion};}\\{\scriptsize order statistics, threshold}};
\node[ex, right=5mm of pre] (feat) {\textbf{24 features}\\{\scriptsize functionals of $(\bH,\Kq)$ only}};
\node[ln, right=5mm of feat] (enc) {\textbf{encoder} $\times 4$\\{\scriptsize cell path $+$ gain-biased}\\{\scriptsize attention, ${\sim}115$k weights}};
\node[ln, right=5mm of enc] (head) {\textbf{head}\\{\scriptsize bounded correction}\\{\scriptsize on $\bw_{\mathrm{clip}}$, zero-init}};
\node[ex, right=5mm of head] (clamp) {\textbf{cut clamp}\\{\scriptsize $\bw=\min(\bff, f^{\uparrow}_{\Kq})$}\\{\scriptsize (Prop.~\ref{prop:clamp})}};
\node[ex, right=5mm of clamp] (fp) {\textbf{recursion \eqref{eq:recursion}}\\{\scriptsize 40 passes}};
\node[io, right=5mm of fp] (out) {$\bp_\theta$};
\draw[fl] (in)--(pre); \draw[fl] (pre)--(feat); \draw[fl] (feat)--(enc);
\draw[fl] (enc)--(head); \draw[fl] (head)--(clamp); \draw[fl] (clamp)--(fp); \draw[fl] (fp)--(out);
\end{tikzpicture}}
\caption{The champion discovered by the campaign. Teal stages are exact algebra
with no trainable parameters; amber stages carry all trained weights.}
\label{fig:pipeline}
\end{figure*}

\subsubsection{Classical scaffolding}
The scaffolding is not the agent's invention. SINR balancing by fixed-point
iteration is classical \cite{yates,foschini,schubert}, and we restate it here
only because what the campaign discovered was where to \emph{put} it. Define
the balancing recursion
\begin{equation}
\bp \;\leftarrow\; \Pmax\,\frac{\bw \odot \bF(\bp)}{\max\big(\bw \odot \bF(\bp)\big)},
\qquad
F_{k,b}(\bp)=\tfrac{\mathcal{B}_{k,b}(\bp)}{\own{k}{b}},
\label{eq:recursion}
\end{equation}
in which $\odot$ is the elementwise (Hadamard) product, $\mathcal{B}_{k,b}$ is
the interference-plus-noise power appearing in \eqref{eq:sinr}, and $\bw>0$ is a
target SINR profile, one entry per user.

\begin{proposition}
\label{prop:balance}
\emph{(i)} $\bp$ is a fixed point of \eqref{eq:recursion} if and only if
$\bgam(\bp)=c\,\bw$ for some $c>0$, with peak power $\Pmax$, where $\bgam(\bp)$
collects the SINRs \eqref{eq:sinr}; \emph{(ii)} for $\bw\equiv\mathbf{1}$,
the all-ones vector, this fixed point is the max-min-rate optimum
$\bp^{*}(\bH)$; \emph{(iii)} the iteration converges to this fixed point from any
positive initialization.
\end{proposition}

\begin{IEEEproof}[Proof sketch]
Part~(i) follows by algebra on \eqref{eq:sinr}. For part~(iii), fix the target
$\btau=c\bw$ and read $\btau\odot\bF(\bp)$ as the power each user would need in
order to meet its target given everyone else's current powers. That map is a
\emph{standard interference function} in the sense of \cite{yates}: it is
positive, it increases when interference increases, and scaling all powers up by
a factor raises the requirement by less than that factor, because the noise term
does not scale. Any such map has a unique fixed point, reached from any positive
start \cite{yates,foschini}. The recursion in \eqref{eq:recursion} differs in
that it renormalizes the target at every step rather than holding it fixed;
convergence of this variant to the SINR-balanced point follows from the
nonlinear Perron--Frobenius results of \cite{schubert}. Part~(ii) is then the
observation that balancing a uniform target equalizes all SINRs, which is what
max-min-rate optimality requires.
\end{IEEEproof}

Let $\bgam_{\mathrm{fp}}=\bgam(\Pmax\mathbf{1})$ denote the full-power SINRs
and $\mathrm{thr}_{\mathrm{fp}}$ their $\Kq$th smallest value. The campaign's
second family constructed from these the \emph{$\Kq$-clipped anchor}, namely the
allocation $\bp_{\mathrm{clip}}$ obtained by driving \eqref{eq:recursion} with
\begin{equation}
w_{\mathrm{clip}}[k,b]=\min\!\big(\gamma_{\mathrm{fp}}[k,b]/\mathrm{thr}_{\mathrm{fp}},\,1\big).
\label{eq:wclip}
\end{equation}
This anchor balances all above-threshold users against one another, while
declining to equalize the network down to those the full-power geometry has
already stranded; it reduces exactly to $\bp^{*}$ at $\Kq=1$.

What the campaign contributed here was placement rather than mathematics. Having
found at experiments~27--29 that the classical solution serves better as an
input feature than as a distillation target, the agent had to compute it inside
the inference budget, and the obvious routes do not fit. The reference algorithm
of \cite{parti} attains its accuracy by solving an exponential conic program at
every iteration, a solve intrinsic to the transform rather than an artifact of
implementation; max-min power control on its own is usually obtained by
bisection, which is likewise per-instance and iterative. The fixed-point form in
\eqref{eq:recursion} is the route that survives: forty passes of elementwise algebra, cheap enough to run per instance
and differentiable, so the same recursion serves both as a feature generator
and, later, as the map from emitted profile to powers.

\subsubsection{Learned components}
One parameter set serves every $K$ in the trained range. The $24$ per-user input features, defined in
full in the repository, span eight blocks: base gains, SINR probes at fixed
allocations, the max-min operating point, global order statistics, the
$\Kq$-clipped balance, set-restricted interference couplings, within-cell
context and $(K,\Kq)$ conditioning. Every one is a closed-form functional of
$(\bH,\Kq)$ alone: none evaluates a rate or the SLqP objective, and the order
statistics they use are taken over fixed, input-only SINR probes rather than
over the model's own allocation, so the features are fixed functions of the
problem instance rather than a partial solution of \eqref{eq:problem} for any
$\Kq>1$. Two blocks address encoder limitations that are structural rather than a matter
of capacity: the SLqP utility depends upon the rate vector only through the
identity of its $\Kq$ smallest entries, an ordinal property normalized pooling
cannot locate, which the order-statistic and set-coupling blocks supply
directly.

The encoder is permutation-equivariant, each round combining a cell-mediated
message-passing path, motivated by the fact that powers enter \eqref{eq:sinr}
only through per-cell totals, with a global attention path whose logits are
biased, per head and per round, by both the
aggressor-to-victim and victim-to-aggressor log-gains. Both directions are
supplied explicitly, being distinct entries of $\bH$: how strongly cell $b'$
interferes with user $(k,b)$ says nothing about how strongly cell $b$ interferes
with the users of $b'$, so neither is recoverable from the other. We emphasize
that the head does not emit powers; it emits a bounded multiplicative correction
upon the anchor profile $\bw_{\mathrm{clip}}$, zero-initialized in the sense that
its weights begin at values for which the correction is unity, so that training
starts at the classical policy and learns only departures from it.

\subsubsection{The clamp and an exactness guarantee}
The emitted profile $\bff$ subsequently passes through the cut clamp, which pins
every above-threshold user to the threshold, and finally through
\eqref{eq:recursion} to yield powers. The clamp
was introduced at experiment~49, at which point the agent observed, and the
following proposition confirms, that it resolves the entire minimum-percentile
column of the evaluation grid exactly: at $\Kq=1$ the model returns the
max-min-optimal allocation, and does so for every value of the trained weights,
so that column cannot regress.
\begin{proposition}
\label{prop:clamp}
For $\Kq=1$, the output of the model satisfies $\bp_\theta(\bH,1)=\bp^{*}(\bH)$
for every value of the trained parameters $\theta$.
\end{proposition}
\begin{IEEEproof}
At $\Kq=1$ the clamp threshold $f^{\uparrow}_{1}$ is the smallest entry of $\bff$,
whence the clamped profile is uniform, $\bw\equiv f^{\uparrow}_{1}\mathbf{1}$.
Since \eqref{eq:recursion} is invariant to a positive rescaling of $\bw$, driving
it with this profile is equivalent to driving it with $\bw\equiv\mathbf{1}$, which
yields $\bp^{*}(\bH)$ by Proposition~\ref{prop:balance}(ii). The resulting
allocation is therefore independent of $\bff$, and hence of $\theta$.
\end{IEEEproof}

Proposition~\ref{prop:clamp} concerns the fixed point of
\eqref{eq:recursion}; the implementation runs a fixed forty passes, reproducing
$\bp^{*}$ to a relative error of $3.2\times10^{-7}$ on held-out data. The reparameterization is also lossless, since any feasible
allocation with one user at $\Pmax$ is the fixed point of its own SINR profile. The network thus
selects only the \emph{shape} of the allocation, and the two failure modes that
afflict direct power outputs, all-off collapse and full-power saturation, are
rendered structurally unreachable.

\subsubsection{Training objective}
The loss used by the champion solution is
\begin{equation}
\mathcal{L} = -\,\frac{f_{\Kq}\big(\mathbf{r}(\bp_\theta(\bH))\big)}
{f_{\Kq}\big(\mathbf{r}(\Pmax\mathbf{1})\big)}
\;+\; \alpha_T\,\mathcal{L}_{\mathrm{distill}},
\label{eq:loss}
\end{equation}
in which $\mathbf{r}(\bp)$ denotes the vector of rates. The first term is the
true SLqP objective, normalized per task by that batch's own full-power SLqP.
The agent's stated rationale, recorded in the log, is that the gradient magnitude of the raw objective scales
with $\Kq$, thereby starving precisely those small-$\Kq$ settings at which the
headroom over full power is largest. This normalization, introduced at
experiment~2, proved to be the single largest gain of the entire campaign. The second term distills, at constant weight $\alpha_T=1.0$, against cached
teacher profiles from short local searches initialized at the anchor
$\bw_{\mathrm{clip}}$, compared one-sidedly in gauge-fixed log-profile
coordinates. \subsubsection{Task sampling and optimization}
At each of $2000$ optimizer steps, eight independent $(K,\Kq)$ tasks are drawn
and their gradients averaged prior to a single Adam update. The size $K$ is
drawn uniformly on $\{1,\dots,10\}$, including the values $\{3,5,7,9\}$ never
present on the evaluation grid, and the percentile \emph{continuously} across
the band, rather than being restricted to the three graded targets. Furthermore, $\Kq=1$ is excluded, since by
Proposition~\ref{prop:clamp} its output is optimal for any parameters. The optimizer is Adam at a peak rate of $2\times10^{-3}$, cosine-annealed over
the budget. The recursion \eqref{eq:recursion} is
differentiated through, so the encoder observes the true sensitivity of the
powers to the emitted profile.

\section{Campaign Results}
\label{sec:results}

\begin{figure}[t]
\centering
\includegraphics[width=\columnwidth]{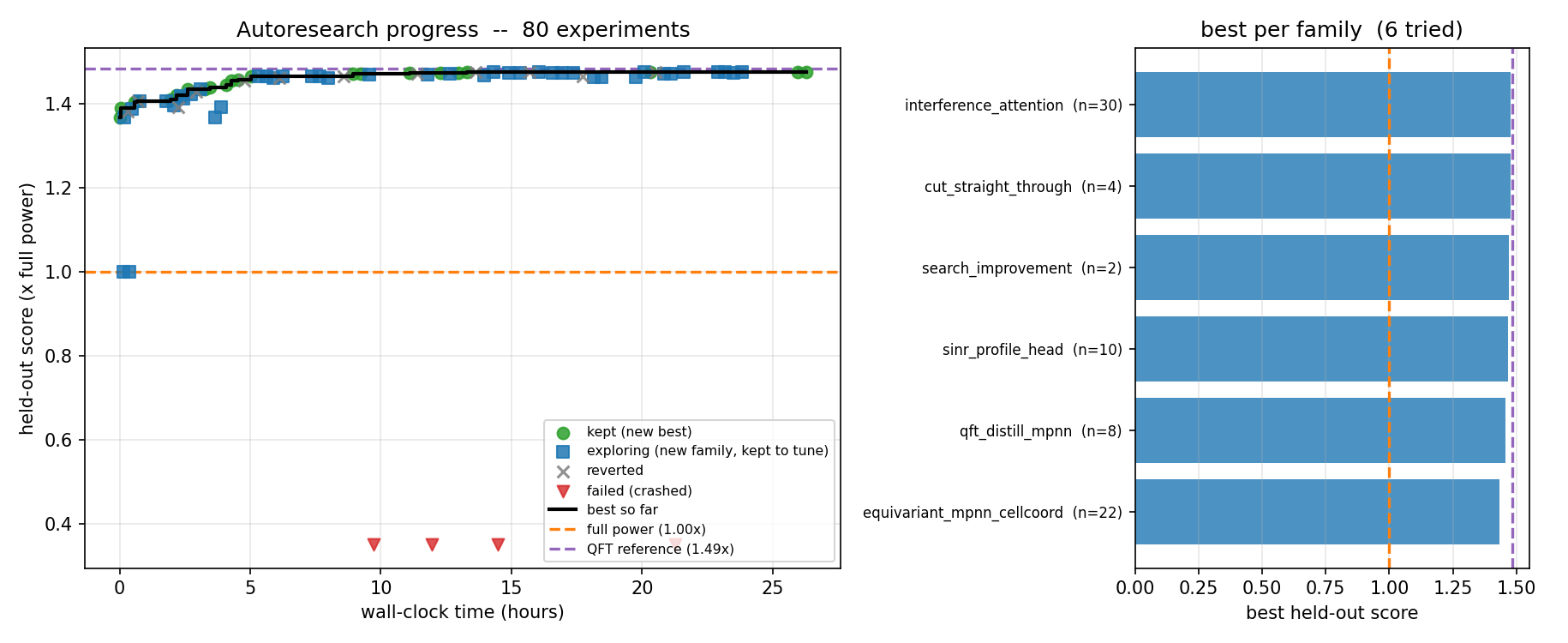}
\caption{Evolution of the campaign. \emph{Left:} held-out score against
wall-clock time, one marker per experiment; the staircase traces the incumbent.
\emph{Right:} best score within each of the six families opened, with the number
of experiments each received; the budget concentrates upon the two most
productive, as the breadth-then-depth protocol mandates.}
\label{fig:progress}
\end{figure}

\begin{table}[t]
\caption{Milestone experiments (banked champions and notable reverts).}
\begin{center}
\footnotesize
\begin{tabular}{@{}rlc@{}}
\toprule
\textbf{Exp} & \textbf{Change} & \textbf{Score} \\
\midrule
--- & full-power floor & 1.0000 \\
1  & equivariant cell-coordinated MPNN & 1.3687 \\
2  & ratio-normalized loss (largest single gain) & 1.3906 \\
27--29 & closed-form $\Kq{=}1$ teacher $\to$ input features & 1.4570 \\
31 & SINR-profile output reparameterization & 1.4656 \\
41--42 & gain-biased bidirectional attention & 1.4725 \\
49 & cut clamp (Prop.~\ref{prop:clamp}); min column exact & 1.4740 \\
52 & $\Kq\ge2$ training law & 1.4756 \\
65 & \emph{QFT itself as distillation teacher (reverted)} & \emph{1.4634} \\
\textbf{81} & \textbf{output-map resolution fix ($W_{\mathrm{scale}}$)} & \textbf{1.4775} \\
\midrule
--- & QFT reference (converged, sample-matched) & 1.4850 \\
\bottomrule
\end{tabular}
\end{center}
\label{tab:milestones}
\end{table}

Fig.~\ref{fig:progress} traces the campaign, and Table~\ref{tab:milestones} records its milestones. The champion scores $1.4775$ against the reference $1.4850$ on the pinned grid: $99.5\%$ of the strongest known benchmark for this strongly NP-hard problem, attained by one forward pass of the network followed by a fixed number of algebraic iterations, with no per-instance optimization. Measured against its own first working architecture at $92.2\%$, the agent closed $94\%$ of the remaining gap over eighty-one unattended cycles. On an Apple M2 Pro, inference over the entire grid takes $2.52$\,s against $1583$\,s for the reference solver on the same machine, a speedup of roughly $600\times$ and well inside the evaluator's ten-second contract. At $\Kq=1$ the emitted allocation is $\bp^{*}$ exactly, on
every pairing of the minimum-percentile row, as Proposition~\ref{prop:clamp} requires.

Three features of the trajectory distinguish autonomous \emph{research} from
autonomous tuning. First, the largest gains were \emph{diagnoses} rather than sweeps: experiment~2
followed from recognizing that the gradient magnitude of the raw objective
scales with $\Kq$, and experiment~81 from identifying the resolution of the
head's output map, rather than descent noise, as the binding constraint.

Second, the negative results are informative, each closed with at least two
independent probes: added capacity failed four times, objective softening three,
percentile re-weighting twice, distillation six. Notably, the last includes distillation from the certified reference solver
itself, at experiment~65, which scored appreciably below the self-trained
student: training across many realizations beats imitating a per-instance
optimizer, the campaign's most interesting negative finding.

Third, each family was opened by a diagnostic finding of its predecessor rather
than by enumeration. The founding message-passing family established that
pointwise models cannot represent the objective's coupling through per-cell
totals. The second found the closed-form max-min solution serves better as an
\emph{input feature} than as a distillation target.
The third reparameterized the output from powers to a target SINR profile,
making the classical policy the zero point of the search space; its plateau, and the argument that the function class could not represent the
required all-pairs comparison, motivated the fourth, attention-based family,
within which the remaining gains were found. Fig.~\ref{fig:progress} shows the budget split.

\subsection{The Strongest Known Benchmark, and What the Plateau Means}

The bar the campaign approached is not one method among several. Part~I of
\cite{parti} develops a second transform, mathematically distinct from the
quadratic one and carrying no guarantee of reaching the same stationary point;
the two nonetheless converge to closely matched values, and \cite{parti}
concludes that improving upon them may require a radically different approach. A
hand-designed self-supervised learner for the same objective likewise fails to
exceed it across the cell-edge band, trailing on six of seven reported settings
\cite{selfsup}. Neither an independent classical route nor an independent
learned one has beaten QFT here, which makes it the strongest known benchmark
for this problem rather than a convenient reference.

That settles what the campaign's final phase means. In its last dozen
experiments every attempt at a qualitatively new mechanism reverted, leaving
only parameter-level refinements to pay. This is a diagnosis, not a shortfall: the search
saturated at the strongest performance anyone has reached on this problem, at
roughly $600\times$ lower inference cost than any method of comparable quality.

The literature supports this reading. \cite{aitelco} reports convergence toward
classical variants on one of three tasks, channel estimation with known
covariance, where LMMSE is the optimal linear estimator and hence itself the best
available; on the other two, where no such solution exists, that framework
produced markedly different algorithms and beat a fine-tuned baseline outright.
The two campaigns together support a sharper statement than either alone:
agentic search recovers the classical solution precisely where the classical
solution is already the best available, and produces novel structure where it is
not. Convergence onto a classical form is therefore diagnostic of
the problem, not a limitation of the method.

One qualification belongs on the record: the hybrid that emerged is shaped by a
problem supplying a closed-form solution at one vertex, by a charter placing
Part~I \cite{parti} in the agent's context, and by an inference contract
penalizing per-instance iteration; a problem without such an anchor should be
expected to yield something else.

\subsection{Scope and Limitations}

Several qualifications belong on the record. All scores are simulator scores:
the evaluator implements the model of \cite{parti}, certified against its
reference implementation, but no measured data is used, and probes within the
$\pm0.0005$ noise band are not claimed as gains. The findings are empirical and
budget-dependent: other limits on training and inference time would favour other
designs, and the campaign is one draw from a stochastic search, as
Section~\ref{sec:conclusion}-A records. Our target was the cell-edge band
$(0,0.25]$; other metrics, sum-rate among them, require their own investigation.
The agent's log is self-reported; what renders it trustworthy is the version
history and the pinned evaluator, not the prose.

\section{Conclusions}
\label{sec:conclusion}

The learning-to-optimize literature has established that trained networks can
replace iterative solvers for interference management \cite{sun,eisen,shen}. In every case the architecture, the objective and the training recipe were the
product of months of human judgement. That layer constrains not only how quickly
the field takes on a new problem but how reliably: a poorly specified reward
makes a sound formulation look like a dead end, such false negatives are seldom
published, and method choice tracks the designer's training as much as the
problem's structure. This
paper removes it. On a network-level problem that is non-convex, non-smooth and
strongly NP-hard, an agent given only an immutable evaluator and a research charter designed a
system reaching within half a percent of the strongest known benchmark at roughly
$600\times$ lower inference cost, with one parameter set for every network size
and percentile target, together with a provable exactness guarantee and a log that is itself a result:
six design families carried to a fixed budget under a metric no participant
could edit is a survey of the space, not a record of one path through it. The human contribution reduced to stating the problem
and building an honest judge. We expect the pattern to hold across most RRM problems: wherever a fixed
simulator can serve as an impartial judge and a scalar metric captures the
engineering objective, the design of learned wireless algorithms becomes an
unattended and auditable process.

\subsection{Lessons Learned}

Three observations may assist others. The first concerns the division of labour among models: Claude Opus~5 built the
evaluator, seed script and charter, Claude Code served as the harness, and
Claude Sonnet~5 drove the outer loop. The initial design demands the stronger
model; the iterative loop, once the protocol is fixed, is well served by a
faster one.

The second concerns randomness, which plays a larger role than the final
artifact suggests. Which hypothesis the agent states next is sampled, and one
early divergence redirects the rest of the search. A colleague running a
comparable charter and prompts did not arrive at the same architecture and
achieved a lower score; a single campaign should not be read as identifying a
unique or globally best design.

The third concerns the charter, which we wrote in firm, authoritative and
unambiguous language, stating not merely what the agent should attempt but what
it must not do. This proved decisive: an agent running unattended for tens of hours acts upon
whatever latitude the charter leaves it, and hedged phrasing invites
reinterpretation precisely when no human is present to correct it. The family cap, the falsifier requirement and the prohibition upon editing the
evaluator were expressed as obligations.

\end{document}